\documentclass[letterpaper, 10 pt, conference]{ieeeconf}  % Comment this line out if you need a4paper

\IEEEoverridecommandlockouts                              % This command is only needed if 
\usepackage{amsmath} % assumes amsmath package installed
\usepackage{amssymb}  % assumes amsmath package installed
\usepackage{amsthm}
\usepackage{xcolor}
\usepackage{algorithm}
\usepackage{algpseudocode}
\usepackage{algorithmicx}
\usepackage{bbm}
\usepackage{graphicx}
\usepackage{comment}
\usepackage{subcaption}
\usepackage{dsfont}

\usepackage{etoolbox}
\makeatletter
\patchcmd{\@addpunct}{:}{\space}{}{}
\makeatother

\makeatletter
\renewenvironment{proof}[1][\proofname]{\par
  \pushQED{\qed}%
  \normalfont \topsep6\p@\@plus6\p@\relax
  \trivlist
  \item[\hskip\labelsep\itshape #1.\@addpunct{\space}]\ignorespaces
}{%
  \popQED\endtrivlist\@endpefalse
}
\makeatother

\title{\LARGE \bf
Transfer in Sequential Multi-armed Bandits via Reward Samples 
}

\author{NR Rahul and Vaibhav Katewa% <-this % stops a space
\thanks{NR Rahul is with the Department of Electrical  Communication Engineering (ECE) at the Indian Institute of Science, Bengaluru, India. Email: 
        {\tt\small rahulnr@iisc.ac.in}}%
\thanks{Vaibhav Katewa is with the Robert Bosch Center for Cyber-Physical Systems and the Department of ECE at the Indian Institute of Science, Bengaluru, India. Email:
        {\tt\small vkatewa@iisc.ac.in}}%
}

\newtheorem*{remark}{Remark}
\newtheorem{assumption}{Assumption}
\newtheorem{theorem}{Theorem}

\newtheorem{lemma}{Lemma}

\begin{document}

\maketitle
\thispagestyle{empty}
\pagestyle{empty}

%%%%%%%%%%%%%%%%%%%%%%%%%%%%%%%%%%%%%%%%%%%%%%%%%%%%%%%%%%%%%%%%%%%%%%%%%%%%%%%%
\begin{abstract}

We consider a sequential stochastic multi-armed bandit problem where the agent interacts with bandit over multiple episodes. The reward distribution of the arms remain constant throughout an episode but can change over different episodes. %The goal of the agent is to minimize the cumulative regret over $J$ episodes. This setting is useful in applications like recommender systems where the changing reward distributions capture the behavior of user preferences changing over time.
We propose an algorithm based on UCB to transfer the reward samples from the previous episodes and improve the cumulative regret performance over all the episodes. We provide regret analysis and empirical results for our algorithm, which show significant improvement over the standard UCB algorithm without transfer.     

\end{abstract}

%%%%%%%%%%%%%%%%%%%%%%%%%%%%%%%%%%%%%%%%%%%%%%%%%%%%%%%%%%%%%%%%%%%%%%%%%%%%%%%%
\section{INTRODUCTION}

The Multi-armed Bandit (MAB) problem \cite{bubeck2012regret,lattimore2020bandit,robbins1952some} is a popular sequential decision-making problem where an agent interacts with the environment by taking actions at every time step and in return gets a random reward. The goal of the agent is to maximize the average reward received. Recently, there has been a lot of interest in the application of the MAB problem in the context of online advertisements and recommender systems\cite{9185782, silva2022multi}. One of the problems highlighted in \cite{silva2022multi} is the user cold start problem, which is the inability of a recommender system to make a good recommendation for a new user in absence of any prior information. In this scenario, it is useful to transfer knowledge from other related users in order to make better initial recommendations to the new user. In the context of a MAB problem, transfer learning uses knowledge from one bandit problem in order to improve the performance of another related bandit problem \cite{lazaric2013sequential, shilton2017regret}. In particular, it helps to accelerate learning and make better decisions quickly.

In this paper, we consider a sequential stochastic MAB problem where the agent interacts with the environment sequentially in episodes (similar to \cite{lazaric2013sequential}), where different episodes are synonymous with different tasks or different bandit problems. The reward distributions of the arms remain constant throughout the episode but change over different episodes. This scenario is useful, for instance, in recommender systems where the reward distributions of recommended items change in order to capture the changing user preferences over time. The goal is to leverage the knowledge from previous episodes in order to improve the performance in the current episode, thereby leading to an overall performance improvement. Towards this, we use reward samples from previous episodes to make decisions in the current episode. Our algorithm is based on the UCB algorithm for bandits \cite{auer2002finite}.

%Specifically, we transfer reward samples from previous episodes in order to build an estimate for the mean of the arms. 
%Using the estimates we act optimistically as in UCB\cite{auer2002finite} to pick the best arm.
\noindent \textbf{Related Work:} Transfer learning in the context of MAB has been studied in the framework of Multi-task learning \cite{lazaric2013sequential,zhang2017transfer,deshmukh2017multi,liu2018transferable} and Meta-learning \cite{cella2020meta,cella2021multi, azizi2023meta}. In Multi-task learning, the set of tasks are fixed and they are repeatedly encountered by the learning algorithm whereas in meta-learning, the algorithm learns to adapt to a new task after learning from a few tasks drawn from the same task distribution. For instance, the authors in \cite{lazaric2013sequential} consider a finite set of bandit problems which are encountered repeatedly over time. In contrast, we consider an infinite set of bandit problems but assume that the problems are ``similar" (we define the notion of ``similarity" later). The idea of transferring knowledge using samples is used in the SW-UCB algorithm in \cite{garivier2011upper}, but it suffers from the notion of negative transfer, where knowledge transfer can degrade the performance. In contrast, our algorithm facilitates knowledge transfer while guaranteeing that there is no negative transfer.

 %In order to avoid negative transfer we use another estimate of the means using the reward samples from the current episode and choose this estimates optimistic reward if it's better than the other one. This approach to avoid negative transfer is similar to \cite{lazaric2013sequential} but the difference is they consider a finite set of bandits and their performance is affected when the bandits are close enough in terms of their means. The idea of transfering samples is employed in SW-UCB\cite{garivier2011upper} but it suffers from negative transfer. In short, we transfer the reward samples from previous episodes to improve performance at the same time guaranteeing there is no negative transfer. Some of the other frameworks that have transfer aspects in MABs are Multi-task learning\cite{lazaric2013sequential}\cite{zhang2017transfer}\cite{deshmukh2017multi}\cite{liu2018transferable} and Meta-learning\cite{cella2020meta}\cite{cella2021multi}\cite{azizi2023meta}. The difference between the two is that in Multi-task learning the set of tasks are fixed and repeatedly encountered by the learning algorithm whereas in meta-learning the algorithm learns to adapt to the new task after learning from few tasks drawn from the same task distribution.

The main contributions of the paper are:\\
\noindent(i) We develop an algorithm based on UCB to transfer knowledge using the reward samples from the previous episodes in a sequential stochastic MAB setting. Our algorithm has a better performance compared to UCB with no transfer.\\
\noindent(ii) We provide the regret analysis for the proposed algorithm and our regret upper bound explicitly captures the performance improvement due to transfer.\\
(iii) We show via numerical simulations that our algorithm is able to effectively transfer knowledge from previous episodes. 
%algorithm experiments to validate both theoretical and empirical performance improvement.

\noindent \textbf{Notations:} $\mathds{1}\{E\}$ denotes the indicator function whose value is $1$ if the event (condition) $E$ is true, and $0$ otherwise. $\emptyset$ denotes null set.   

\section{PRELIMINARIES AND PROBLEM STATEMENT}

We consider the Multi-Armed Bandit problem with $K$ arms and $J$ episodes. The length of each episode is $n$. Define $[K] \triangleq \{1,2,\cdots,K\}$ and $[J] \triangleq \{1,2,\cdots,J\}$. At any given integer time $t>0$, one among the $K$ arms is pulled and a random reward is received. Let $I_t \in [K]$ and $r_{I_t}$, denote the arm pulled at time $t$ and the corresponding random reward, respectively. We assume that $r_{I_t}\in [0,1]$ and the rewards are independent across time and across all arms. In any given episode, the distributions of the arms do not change. However, they are allowed to be different over different episodes.

% Each arm is pulled one at a time sequentially and a reward is received. The rewards are drawn independently from probability distributions $V_k$ corresponding to the arm $k$. The distributions change after every $n$ time step and the entire length of time until the next change forms an episode. The total number of episodes is denoted by $J$. \underset{k\in [K]}{\arg \max}

Let $\mu_{k}^{j}$ be the mean reward of arm $k$ in episode $j$. Let $\mu^{j} \triangleq [\mu_1^{j},\mu_2^{j},\cdots, \mu_K^{j}]^T$ denote the vector containing the mean rewards of all arms for episode $j$. Further, let $k^{j}_{*} \in \mathcal{A}^j \triangleq \underset{k\in [K]}{\arg \max} \{\mu_{k}^{j}\}$ and
$\mu^{j}_{*} = \max\limits_{k\in [K]} \{\mu_{k}^{j}\}$ denote an optimal arm\footnote{There may be more than one optimal arms which have equal maximum mean rewards.} in episode $j$ and it's mean reward, respectively. Define $\Delta_{k}^{j} = \mu_{*}^{j}-\mu_{k}^{j}>0$ as the sub-optimality gap of arm $k \notin \mathcal{A}^j$ in episode $j$. Note that the mean rewards of the arms are unknown.

We assume that the episodes in the MAB problem are related in the sense that the mean rewards of the arms across episodes do not change considerably. We capture this by the following assumption.

\begin{assumption}\label{assump1}
    We assume that $||\mu^{j_1} - \mu^{j_2}||_{\infty} \leq \epsilon $ for any $j_1,j_2 \in [J]$, where the parameter $0< \epsilon \leq 1$ is assumed to be known.
\end{assumption} 
This assumption implies that for each arm, the mean rewards across all episodes do not differ by more than $\epsilon$. In applications like online advertising and recommender systems, the user preferences change over time only gradually, and therefore, the parameter $\epsilon$ can be used to capture this behaviour. 
% We define $\mu_{k}^{j}$ as the mean of distribution $V_k$ in the $j$th episode. We assume that $|\mu_{k}^{j}-\mu_{k}^{j\prime}|\leq\epsilon,\forall j,j\prime \in [J]$ and $\epsilon>0$ is a positive constant. We define $\mu_{j}^{*} = \max\limits_{k}\mu_{k}^{j} $, optimal arm $k_{j}^{*} = \arg\max\limits_{k}\mu_{k}^{j}$ and $\Delta_{k}^{j} = \mu_{j}-\mu_{k}^{j}$.

Let $N_k^j(t)$ denote the number of pulls of arm $k$ in the time interval $[(j-1)n+1,t]$. Thus, $N_k^j(t)$ counts the number of times arm $k$ is pulled from the beginning of episode $j$ until time $t$. Note that for episode $j$, the allowable values of $t$ in $N_k^j(t)$ are $[(j-1)n+1,nj]$. Further, let $S_k(t)$ denote the number of pulls of arm $k$ in the time interval $[1,t]$. Thus, $S_k(t)$ counts the number of times arm $k$ is pulled from the beginning of episode $1$ until time $t$. For example, if $n=5$ and $j=2$, then $N_k^2(8)$ counts the number of times arm $k$ is pulled in time instants $6,7,$ and $8$. Further, $S_k(8)$ counts the number of times arm $k$ is pulled in the interval $[1,8]$.

The goal of the agent is to decide which arm to pull (what should be the value of $I_t$) at any given time $t$ based on the information $\{r_{I_1}, r_{I_2}, \cdots, r_{I_{t-1}}\}$ in order to maximize the average reward over all episodes. This is captured by the pseudo-regret $R_J$ of the MAB problem over $J$ episodes:

\begin{align} \label{eq_regret}
R_J &= \sum\limits_{j=1}^{J} \mathbb{E}\left[\sum\limits_{t=(j-1)n+1}^{jn}(r_{k_{*}^{j}}-r_{I_t})\right]\nonumber\\
&= \sum\limits_{j=1}^{J} \left(n\mu_{*}^{j}-\mathbb{E}\left[\sum\limits_{t=(j-1)n+1}^{jn}\mu^j_{I_t}\right]\right)\nonumber\\
&= \sum\limits_{j=1}^{J}\sum\limits_{k=1}^K \Delta_{k}^{j}\mathbb{E}[N_k^j (jn)],
\end{align}

\noindent where the last equality follows since $\sum_{k=1}^{K}N_k^j(jn) = n$ for any $j\in[J]$. Thus, the goal is to make decisions $\{I_t: 1 \leq t\leq nJ \}$ to minimize the regret in \eqref{eq_regret}.

In this paper, we exploit the relation among the mean rewards of arms in different episodes (c.f. Assumption \ref{assump1}) in order to minimize the regret $R_J$. This is achieved by reusing (transferring) reward samples from previous episodes to make decisions in the current episode. We describe the approach and the proposed algorithm in detail in the next section.

\section{ALL SAMPLE TRANSFER UCB (AST-UCB)}

Our approach of reusing samples from previous episodes builds on the standard UCB algorithm for bandits. In this section, we first describe the UCB algorithm and then our proposed algorithm, which we call All Sample Transfer UCB (AST-UCB).

\subsection{UCB Algorithm \cite{auer2002finite}}
Intuitively, the arm-pulling decisions should be made on the reward samples obtained from each arm. Since the mean rewards of the arms are unknown, the UCB algorithm computes their sample-average estimates and the corresponding confidence intervals. Then, based on the principle of optimism in the face of uncertainty, the upper (maximum) value in the confidence interval of each arm is treated as the optimistic mean reward of that arm. Then, the arm with the highest optimistic mean reward is pulled.

As time progresses and more reward samples are received, the estimates become better and the confidence intervals become smaller. Thus, the upper value in the confidence interval approaches the true mean. Eventually, the optimistic mean reward of the optimal arm becomes larger than all other sub-optimal arms, and thereafter, only the optimal arm is pulled. 

The standard UCB algorithm is used when the arm distributions are assumed to be the same at all times. However, in our setting, the distributions change over episodes. Therefore, one approach would be to implement UCB algorithm separately in each episode by using only the samples of that particular episode. In other words, the UCB algorithm is restarted at the beginning of every episode and it uses only the reward samples received during the current episode. We call this approach as No Transfer UCB (NT-UCB) algorithm. Next, we explain NT-UCB algorithm for episode $j$.

Let $\hat{\mu}_{1k}^{j}(t)$ denote the sample-average estimate of the mean reward of arm $k$ at time $t$, and is computed as:
{\begin{align}
\label{eq:estimate1}
\hat{\mu}_{1k}^{j}(t) = \:\:\frac{\sum\limits_{\tau = (j-1)n+1}^{t}r_{I_\tau}\mathds{1}\{I_\tau = k\}}{\max\{1,N_k^j(t)\}},
\end{align}}

\noindent where $N_k^j(t)$ denotes the number of times arm $k$ is pulled until time $t$ since the beginning of episode $j$. Next, we compute the optimistic mean reward corresponding to $\hat{\mu}_{1k}^{j}(t)$. For this, we require the following result.

\begin{lemma}\label{lemma1}
Let $\alpha>1$. For episode $j$, time $t\in [(j-1)n+1,jn]$ and  arm $k$, with probability at least $1-\frac{2}{(t-(j-1)n)^{\alpha}},$ the following equation is satisfied 
\begin{align}\label{eq:lemma1}
|\hat{\mu}_{1k}^{j}(t) - \mu_k^j| \leq p_{1k}^j(t)\triangleq \sqrt{\frac{\alpha\log{(t-(j-1)n)}}{2N_k^{j}(t)}}.
\end{align}
\end{lemma}
\begin{proof}
The rewards are independent random variables with support $[0,1]$. Using Hoeffding's inequality\cite{hoeffding1994probability} for estimate $\hat{\mu}_{1k}^{j}(t)$, we get 
\begin{align*}
\text{Pr}\{|\hat{\mu}_{1k}^{j}(t)-\mu_k^j| \geq \delta \} \leq 2\exp(-2N_k^{j}(t){\delta}^2).
\end{align*}
Setting $\delta = \sqrt{\frac{\alpha\log(t-(j-1)n)}{2N_k^{j}(t)}}$ for $N_k^{j}(t)\geq1$, the lemma follows.
\end{proof}

\begin{comment}
Note that the corresponding confidence interval for the mean reward is given as
\begin{align*}
\left[\hat{\mu}_{1k}^{j}(t) -c_{1k}^j(t),\hat{\mu}_{1k}^{j}(t) +c_{1k}^j(t)\right],
\end{align*}
where $c_{1k}^j(t) = \sqrt{\frac{\alpha\log{(t-(j-1)n)}}{2N_k^{j}(t)}}$ and the mean reward lies in this interval with probability at least $1-\frac{2}{(t-(j-1)n)^{\alpha}}$. 
\end{comment}
\begin{comment}
Let $p_{1k}^{j}(t)$ denote the optimistic reward for estimate $\hat{\mu}_{1k}^{j}(t)$, then we have
\begin{align}
p_{1k}^{j}(t) = \hat{\mu}_{1k}^{j}(t) +\sqrt{\frac{\alpha\log{(t-(j-1)n)}}{2N_k^{j}(t)}}.
\end{align}
Similarly, we define the minimum reward in the confidence interval as,
\begin{align}
q_{1k}^{j}(t) \triangleq \hat{\mu}_{1k}^{j}(t) -\sqrt{\frac{\alpha\log{(t-(j-1)n)}}{2N_k^{j}(t)}}.
\end{align}
\end{comment}

Using Lemma \ref{lemma1}, we form a confidence interval for mean reward $\mu_k^j$ using the estimate $\hat{\mu}_{1k}^{j}(t)$ at time $t$ in episode $j$ as
\begin{align*}
D_1^j(t) = [\hat{\mu}_{1k}^{j}(t)-p_{1k}^j(t), \hat{\mu}_{1k}^{j}(t)+p_{1k}^j(t)].
\end{align*}

Next, the NT-UCB algorithm pulls the arm with maximum optimistic reward:
{\begin{align*}
I_t = \underset{k\in [K]}{\arg \max}\left\{\hat{\mu}_{1k}^{j}(t-1)+p_{1k}^{j}(t-1)\right\}.
\end{align*}}
The above steps are repeated until the end of episode $j$. Next, we provide an upper bound on the pseudo-regret of the NT-UCB algorithm.

\begin{lemma}
\label{lemma:ucbRegret}
Let $\alpha>1$. The pseudo-regret of NT-UCB satisfies 
{\begin{align}
\label{eq:ucbRegret}
R_J &\leq \sum\limits_{k=1}^{K}\left[2\alpha \log{n}\bigg(\sum\limits_{{\substack{k=1 \\ \Delta_k^j > 0}}}^{J}\frac{1}{\Delta_k^j}\bigg)+J_k\frac{\alpha+1}{\alpha-1}\right],\\
&\text{where} \quad J_k=  \sum\limits_{k=1}^{J} \Delta_k^j\nonumber
\end{align}}
\end{lemma}
\begin{proof}
    An upper bound on the regret over all episodes is obtained by adding the per-episode regret bound of the standard UCB algorithm, which is given as \cite{bubeck2012regret}\footnote{The second term in \eqref{eq:ucbRegret} differs from the corresponding term mentioned in \cite{bubeck2012regret}, since additional union bounds are used to obtain the result in \cite{bubeck2012regret}.}
    {\begin{align*}
    R_j \leq \sum\limits_{{\substack{k=1 \\ \Delta_k^j > 0}}}^{K} \left(\frac{2\alpha\log{n}}{\Delta_k^j}+\frac{\alpha+1}{\alpha-1}\Delta_k^j\right).
    \end{align*}}
    The result then follows.
    % The overall regret is the sum of the regrets of all the episodes,
    % {\begin{align*}
    % R_J &\leq \sum\limits_{j=1}^{J} R_j\\
    % &=\sum\limits_{j=1}^{J} \sum\limits_{k=1}^{K}\left(\frac{2\alpha \log{n}}{\Delta_k^j}+3\Delta_k^j\right)
    % \end{align*}}
\end{proof}

%In the next section, we will describe how the algorithm is modified in order to reuse the reward samples from the previous episodes.

\subsection{AST-UCB Algorithm}
For any particular episode, the NT-UCB algorithm mentioned above uses samples only in that episode to compute the estimates. However, as per Assumption \ref{assump1}, the mean rewards across the episodes are related, and therefore, reward samples in previous episodes carry information about the mean reward in the current episode. In order to capture this information, we construct an auxiliary estimate (in addition to the UCB estimate) that uses the reward samples from the beginning of the first episode. Then, we combine these two estimates to make the decisions. Next, we describe this approach for episode $j$.

% The disadvantage of this approach is the information present in the reward samples(in the form of bias) from previous episodes is lost.

Let $\hat{\mu}_{2k}(t)$ denote the auxiliary sample-average estimate of the mean reward of arm $k$ at time $t$, computed as:
{\begin{align}
\label{eq:estimate2}
\hat{\mu}_{2k}(t) = \:\frac{\sum\limits_{\tau = 1}^{t}r_{I_\tau}\mathds{1}\{I_\tau = k\}}{\max\{1,S_k(t)\}},
\end{align}}

\noindent where $S_k(t)$ denotes the number of times arm $k$ is pulled
until time $t$ since the beginning of episode $1$. 
Note that estimate $\hat{\mu}_{2k}(t)$ captures the information of reward samples of arm $k$ from all previous episodes \footnote{An alternate strategy would be to construct the auxiliary estimate from a fixed number of previous episodes. However, our strategy is better since the confidence interval corresponding to estimate \eqref{eq:estimate2} is always better than this alternate strategy.}.
Next, we compute the optimistic mean reward corresponding to $\hat{\mu}_{2k}(t)$. For this, we require the following result.
\begin{lemma}\label{lemma3}
Let $\alpha>1$. For episode $j$, time $t\in [(j-1)n+1,jn]$ and  arm $k$, with probability at least $1-\frac{2}{(t-(j-1)n)^{\alpha}},$ the following equation is satisfied 
\begin{align}\label{eq:lemma3}
|\hat{\mu}_{2k}(t) &- \mu_k^j| \leq p_{2k}^j(t) \triangleq \sqrt{\frac{\alpha\log{(t-(j-1)n})}{2S_k(t)}}+U_k^j(t)\epsilon, \\
&\text{where}\quad U_k^{j}(t) = \frac{S_k(t)-N_k^{j}(t)}{S_k(t)}\nonumber.
\end{align}
\end{lemma}

\begin{proof}
The rewards are independent random variables with support $[0,1]$. Using McDiarmid's inequality\cite{mcdiarmid1989method} for estimate $\hat{\mu}_{2k}(t)$, we get 
\begin{align*}
\text{Pr}\{\lvert \hat{\mu}_{2k}(t)-\mathbb{E}[\hat{\mu}_{2k}(t)]\rvert \geq \delta \} \leq 2\exp(-2S_k(t){\delta}^2).
\end{align*}
Setting $\delta = \sqrt{\frac{\alpha\log((t-(j-1)n))}{2S_k(t)}}$ for $S_k(t)\geq1$, we get  
\begin{align*}
&\text{Pr}\left\{|\hat{\mu}_{2k}(t)-\mathbb{E}[\hat{\mu}_{2k}(t)]| \geq \sqrt{\frac{\alpha\log(t-(j-1)n)}{2S_k(t)}}\right\}\\
&\leq \frac{2}{(t-(j-1)n)^\alpha}.
\end{align*}
Hence, with probability at least $1-\frac{2}{(t-(j-1)n)^{\alpha}}$, the following holds
\begin{align}
\label{eq:lemma3_conf}
&|\hat{\mu}_{2k}(t)-\mathbb{E}[\hat{\mu}_{2k}(t)]| \leq \sqrt{\frac{\alpha\log(t-(j-1)n)}{2S_k(t)}}.
\end{align}
Next, we bound $\mathbb{E}[\hat{\mu}_{2k}(t)]$ for $S_k(t) \geq 1$, $t\in[(j-1)n+1,jn]$:
\begin{align}
\label{eq:expect_bound1}
\mathbb{E}[\hat{\mu}_{2k}(t)]&= \frac{\sum\limits_{l=1}^{j-1}N_k^l(ln)\mu_k^l+N_k^j(t)\mu_k^j}{S_k(t)},\nonumber\\
&=\mu_k^j+\frac{\sum\limits_{l=1}^{j-1}N_k^l(ln)(\mu_k^l-\mu_k^j)}{S_k(t)},\nonumber\\
&\leq \mu_k^j+\frac{(S_k^j(t)-N_k^j(t))\epsilon}{S_k(t)},\nonumber\\
&=\mu_k^j+U_k^j(t)\epsilon,
\end{align}
where the inequality follows from $\mu_k^l-\mu_k^j\leq\epsilon$ (Asssumption \ref{assump1}). Similarly, using $\mu_k^l-\mu_k^j\geq-\epsilon$ (Asssumption \ref{assump1}), we get 
\begin{align}
\label{eq:expect_bound2}
\mathbb{E}[\hat{\mu}_{2k}(t)]\geq\mu_k^j-U_k^j(t)\epsilon.
\end{align}
Conditions \eqref{eq:expect_bound1} and \eqref{eq:expect_bound2} yield $|\mathbb{E}[\hat{\mu}_{2k}(t)]|\leq\mu_k^j+U_k^j(t)\epsilon$. Using this in \eqref{eq:lemma3_conf}, we get the result in \eqref{eq:lemma3}.
% have
% {\small\begin{align*}
% |\hat{\mu}_{2k}(t)-\mu_k^j| \leq \sqrt{\frac{\alpha\log((t-(j-1)n))}{2S_k(t)}}+U_k^j(t)\epsilon.
% \end{align*}}
% Hence the lemma follows.
\end{proof}
Using Lemma \ref{lemma3}, we form a confidence interval for mean reward $\mu_k^j$ using the estimate $\hat{\mu}_{2k}(t)$ at time step $t$ in episode $j$ as
\begin{align*}
D_2^j(t) = [\hat{\mu}_{2k}(t)-p_{2k}^j(t), \hat{\mu}_{2k}(t)+p_{2k}^j(t)].
\end{align*}

Next, we present two key steps of the AST-UCB algorithm.

(i) Combine the optimistic rewards of the two estimates $\hat{\mu}_{1k}^{j}(t)$ and $\hat{\mu}_{2k}(t)$ given in \eqref{eq:estimate1} and \eqref{eq:estimate2} as:
\begin{align}
\label{eq:astUcb_opRewrd}
q_k^j(t) = \min \{\hat{\mu}_{1k}^j(t)+p_{1k}^j(t),\hat{\mu}_{2k}(t)+p_{2k}^j(t) \}.
\end{align}

(ii) Pull arm
\begin{align*} 
I_t = \underset{k\in [K]}{\arg \max}\{q_k^j(t-1)\}.
\end{align*}
The above steps are repeated until the end of episode $j$. All the steps of AST-UCB are given below in Algorithm \ref{alg1}.

\begin{algorithm}
\caption{AST-UCB}\label{alg1}
\begin{algorithmic}
\Require Episode length $n$, Number of episodes $J$, Parameters $\alpha$, $\epsilon$ and Number of arms $K$
\For {episode $j = 1,2,...,J$}
    \For {$t = (j-1)n+1,\cdots,(j-1)n+K$}
        \State $I_t = t-(j-1)n$ (Pull each arm once)
    \EndFor
    \For {$t = (j-1)n+K+1,\cdots,jn$}
        \State compute $\hat{\mu}_{1k}^j(t-1)$, $p_{1k}^j(t-1)$ using \eqref{eq:estimate1},\eqref{eq:lemma1}
        \State compute $\hat{\mu}_{2k}(t-1)$, $p_{2k}^j(t-1)$ using \eqref{eq:estimate2},\eqref{eq:lemma3}
        \State compute optimistic reward $q_k^j(t-1)$ using \eqref{eq:astUcb_opRewrd}
        \State select arm $I_t = \underset{k\in [K]}{\arg \max} \{q_k^j(t-1)\}$
        \State update  number of pulls $N_k^j(t)$ and $S_k(t)$
    \EndFor
\EndFor
\end{algorithmic}
\end{algorithm}

Next, we explain the motivation for Step (i). We combine the confidence intervals $D_{1k}^{j}(t)$ and $D_{2k}^{j}(t)$ by taking their intersection to get a better confidence interval. Note that by taking the intersection, the new confidence interval $D_{1k}^{j}(t) \cap D_{2k}^{j}(t)$ is always smaller than the original two confidence intervals, as illustrated in Figure \ref{fig:Figure1}. This smaller interval results in a better estimate of $\mu_k^j$. We then pick the optimistic reward in the new confidence interval\footnote{Note that the Step (i) is valid even when $D_{1k}^{j}(t)$ and $D_{2k}^{j}(t)$ do not intersect.}. Further, Step (ii) is similar to the UCB algorithm where we pull the arm with the maximum optimistic reward. The next result presents a bound on the probability of $\mu_k^j$ lying in the new confidence interval (the new confidence interval being non-empty).

\begin{figure}[h]
\includegraphics[width = \columnwidth]{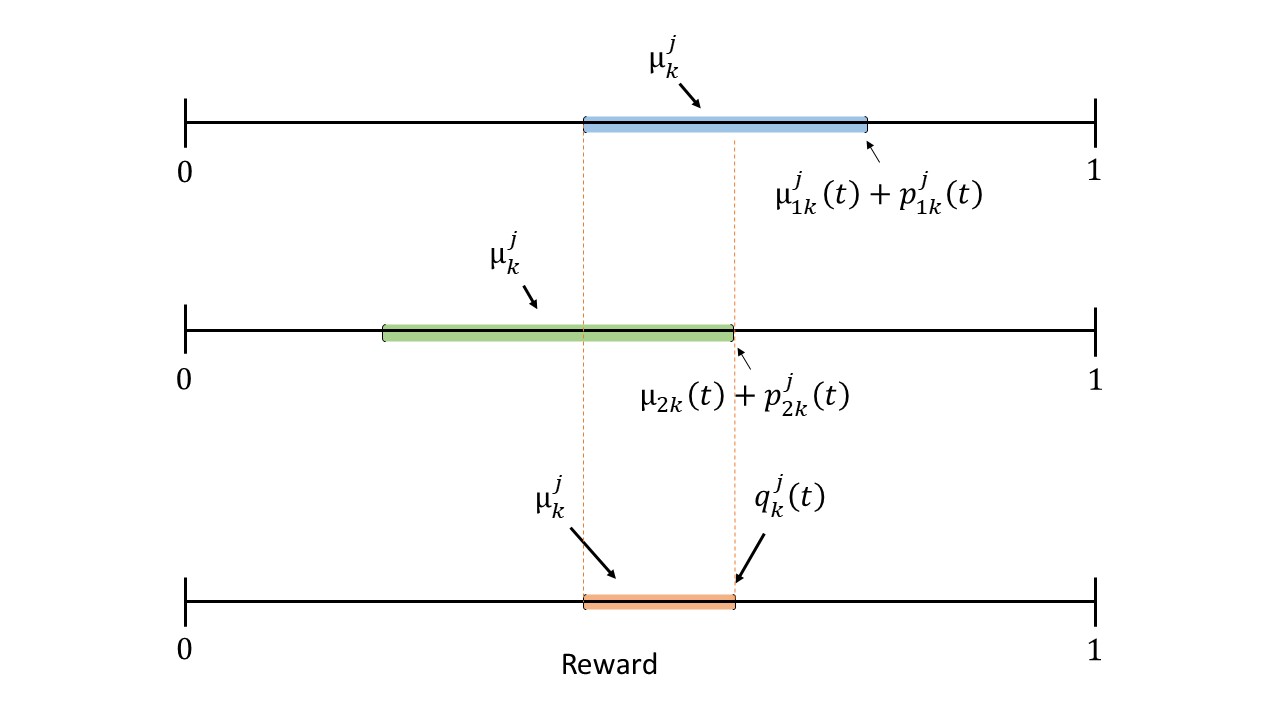}
\caption{The blue and green intervals represent confidence intervals $D_{1k}^{j}(t)$ and $D_{2k}^{j}(t)$ for mean $\mu_k^j$, respectively. The orange interval is the intersection of the two intervals, which is clearly smaller (and hence better). The optimistic reward of the orange interval is given by $q_{k}^j(t)$.}
\label{fig:Figure1}
\end{figure}
 
 %is shown in the Lemma \ref{lemma4}In the next result, we show that with high probability the true mean $ \mu_k^j$ lies in the new confidence interval $D_{1k}^{j}(t) \cap D_{2k}^{j}(t)$,
\begin{comment}
\begin{lemma}
\label{lemma4}
For episode $j\geq1$, time step $t\in [(j-1)n+1,nj]$ and  arm $k\in[K]$, with probability at least $1-\left(\frac{1}{t^{\alpha}}+\frac{1}{(t-(j-1)n)^{\alpha}}\right),$ the following equation is satisfied 
{\begin{align}\label{eq:lemma4}
\mu_k^j \leq \min\{p_{1k}^j(t),p_{2k}^j(t)\}.
\end{align}}
\end{lemma}
\begin{proof}
Consider any fixed value of Episode $j$, time step $t$, and arm $k$. Let event $\mathcal{E}_1 = \{\mu_k^j>r_{1k}^j(t)\}$ and event $\mathcal{E}_2 = \{\mu_k^j>r_{2k}^j(t)\}$, then we have
{\begin{align*}
&\text{Pr}\{\mu_k^j \geq \min(r_{1k}^j(t),r_{2k}^j)(t)\}\\
& = \text{Pr}\{\mathcal{E}_1 \cup \mathcal{E}_2\}\\
&\leq \text{Pr}\{\mathcal{E}_1\} + \text{Pr}\{\mathcal{E}_2\}\\
&=\frac{1}{t^{\alpha}}+\frac{1}{(t-(j-1)n)^{\alpha}}
\end{align*}}
The above inequality is true for any value of episode $j$, time step $t$, and arm $k$. Hence the lemma follows.
\end{proof}
\end{comment}

\begin{lemma}
\label{lemma4}
For episode $j$, time $t\in [(j-1)n+1,jn]$ and  arm $k$, with probability at least $1-\frac{4}{(t-(j-1)n)^{\alpha}},$ the following equations are satisfied 
{\begin{align}\label{eq:lemma4}
&(i)\quad \mu_k^j \in D_{1k}^{j}(t) \cap D_{2k}^{j}(t).\\
&(ii)\quad D_{1k}^{j}(t) \cap D_{2k}^{j}(t) = \emptyset.
\end{align}}
\end{lemma}
\begin{proof}
Define events
$\mathcal{E}_1 = \{\mu_k^j \notin D_{1k}^{j}(t)\}$ and $\mathcal{E}_2 = \{\mu_k^j \notin D_{1k}^{j}(t)\}$.
Then we have
{\begin{align*}
\text{Pr}\{\mu_k^j \notin D_{1k}^{j}(t) \cap D_{2k}^{j}(t)\}
&= \text{Pr}\{\mathcal{E}_1 \cup \mathcal{E}_2\}\\
&\leq \text{Pr}\{\mathcal{E}_1\} + \text{Pr}\{\mathcal{E}_2\}\\
&\leq\frac{4}{(t-(j-1)n)^{\alpha}},
\end{align*}}
where the last inequality follows from Lemmas \ref{lemma1} and \ref{lemma3}. Hence, condition (i) in the lemma follows. Same arguments are valid for condition (ii) as well.
\end{proof}
Note that although the new confidence interval is smaller, Lemma \ref{lemma4} shows that the probability bound of the mean reward belonging to this new interval has reduced as compared to that in \eqref{eq:lemma1} or \eqref{eq:lemma3}. However, we show in Theorem \ref{theorem1} that the negative effect of the reduction of the probability is not significant, and the smaller interval leads to an overall reduction in the regret.

\section{REGRET ANALYSIS}
In this section, we derive the regret of the AST-UCB algorithm and then provide the analysis of the result.
\begin{theorem}\label{theorem1}
Let $\Delta_k^{max} \triangleq \max\limits_{j\in [J]} \: \{\Delta_k^j\} $ and $\Delta_k^{min} \triangleq \min\limits_{j\in [J],\Delta_{k}^{j}> 0} \: \{\Delta_k^j\}$. The pseudo-regret of AST-UCB with $\alpha>1$  and $0\leq\epsilon< \frac{1}{2}\min\limits_{k\in [K]}\{\Delta_k^{min}\}$ satisfies

\begin{align}
\label{eq:astUcb_regret}
R_J&\leq \sum\limits_{k=1}^{K} \Delta_k^{max}\Bigg[\min\bigg\{\sum\limits_{\substack{j=1 \\ \Delta_{k}^{j}> 0}}^{J}\frac{2\alpha\log{(n)}}{(\Delta_k^{j})^2},\frac{2\alpha\log{(n)}}{(\Delta_k^{\min}-2\epsilon)^2}\bigg\}\nonumber\\
&\hspace{3cm}+J\frac{\alpha+3}{\alpha-1} \Bigg].
\end{align}
\end{theorem}

\begin{proof}
Refer to the appendix.
\end{proof}

Next, we compare the regret bounds of our algorithm \eqref{eq:astUcb_regret} and NT-UCB \eqref{eq:ucbRegret}, and highlight the benefit of transfer. The transfer happens due to the first term in \eqref{eq:astUcb_regret}. Hence, we compare the first terms in the regret bounds. %We use the first term from the regret expressions, for comparison since the second terms are comparable to each other. Further, 
To this end, we define the following terms that capture the dependence on $J$:

\begin{align*}
A_k^J = \sum\limits_{\substack{j=1 \\ \Delta_{k}^{j}> 0}}^{J}\frac{\Delta_k^{max}}{(\Delta_k^{j})^2},
B_k^J = \frac{\Delta_k^{max}}{(\Delta_k^{\min}-2\epsilon)^2},
C_k^J = \sum\limits_{\substack{j=1 \\ \Delta_{k}^{j}> 0}}^{J}\frac{1}{\Delta_k^{j}}.
\end{align*}

Several comments are in order. First, observe that, for transfer to be beneficial, we need $\min\{A_k^J,B_k^j\}<C_k^J$. Since $A_k^J\geq C_k^J$, this can happen only if $B_k^J<C_k^J$ . The term $B_k^J$ behaves like a constant as compared to $C_k^J$ which increases as the total number of episodes $J$ increases. Therefore, for some large enough $J^{m}(\epsilon)$, we get $C_k^{J^{m}(\epsilon)} < B_k^{J^{m}(\epsilon)}$ which leads to decrease in the regret as compared to NT-UCB. Second, as $\epsilon$ increases (episodes become increasingly non-related), $J^m(\epsilon)$ increases since more episodes (samples) are required for the transfer to be beneficial. Third, we have logarithmic dependence of episode length $n$ on the regret (which is the case with NT-UCB as well). Fourth, the second term in the regret bound of AST-UCB \eqref{eq:astUcb_regret} is higher than the corresponding term in NT-UCB bound \eqref{eq:ucbRegret} due to the decreased probability bound in Lemma \ref{lemma4} as compared to Lemmas \ref{lemma1} and \ref{lemma3}.

\section{Numerical Simulations}
\begin{figure}[thpb]
\begin{subfigure}{\columnwidth}
    \centering
    \includegraphics[width=\columnwidth]{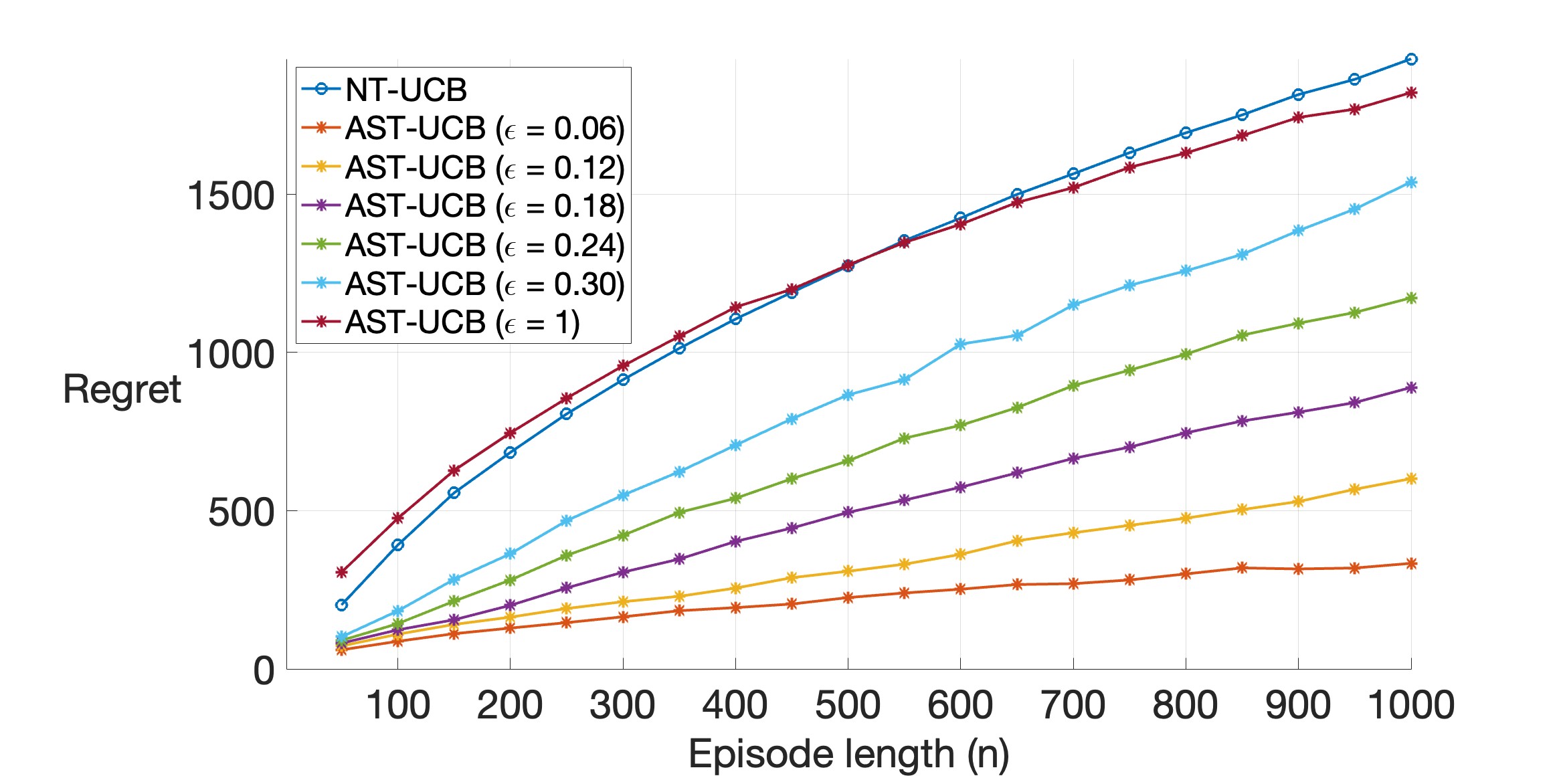}
    \caption{Regret as function of episode length $n$.}
    \label{fig:example_1a}
    \end{subfigure}
\begin{subfigure}{\columnwidth}
    \centering
    \includegraphics[width=\columnwidth]{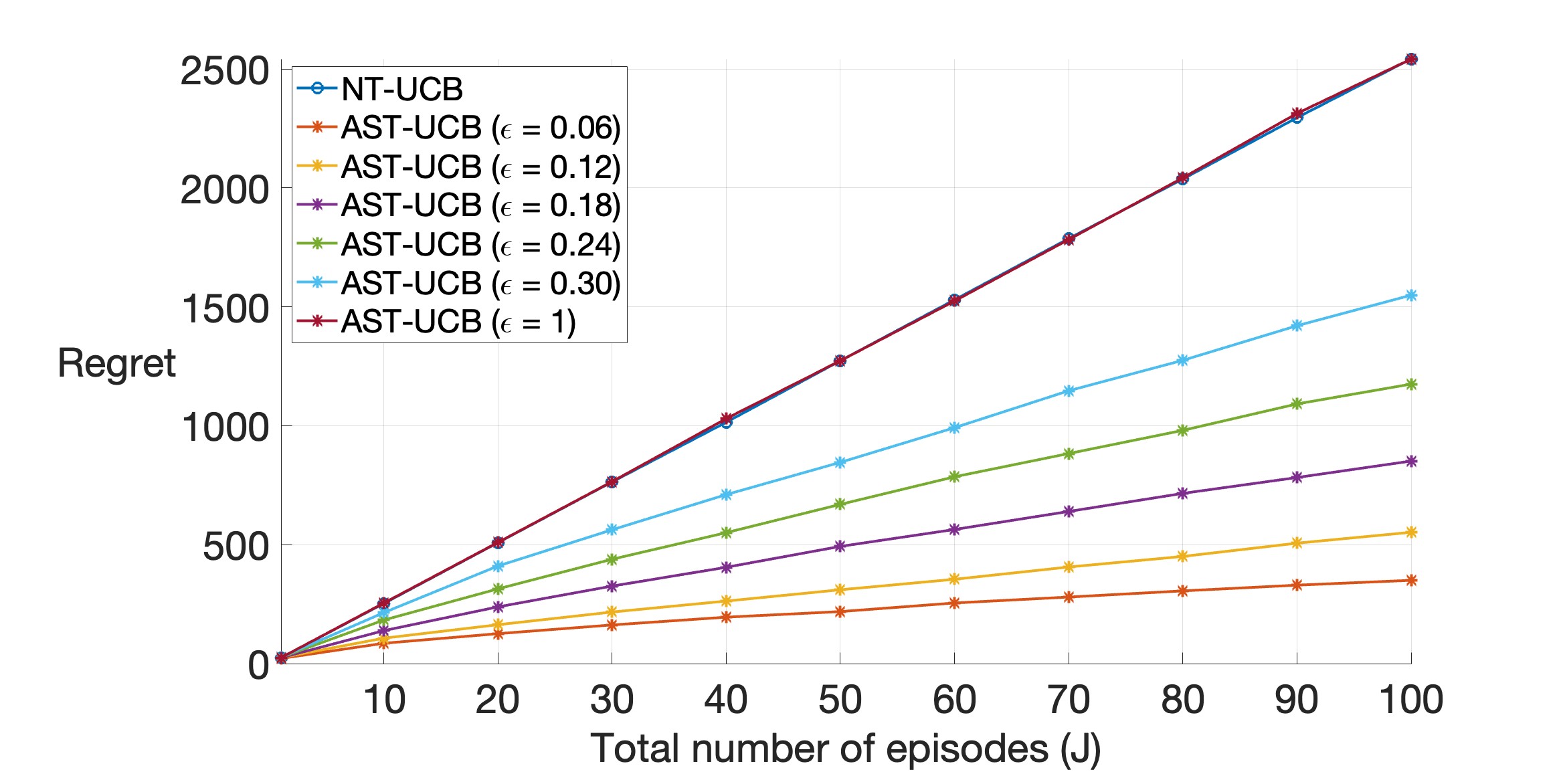}
    \caption{Regret as function of total number of episodes $J$.}
    \label{fig:example_1b}
\end{subfigure}
\caption{Empirical regret of  NT-UCB and AST-UCB for different values of $\epsilon$ for Case I.}
\label{fig:example_1}
\end{figure}

\begin{figure}[thpb]
\begin{subfigure}{\columnwidth}
    \centering
    \includegraphics[width=\columnwidth]{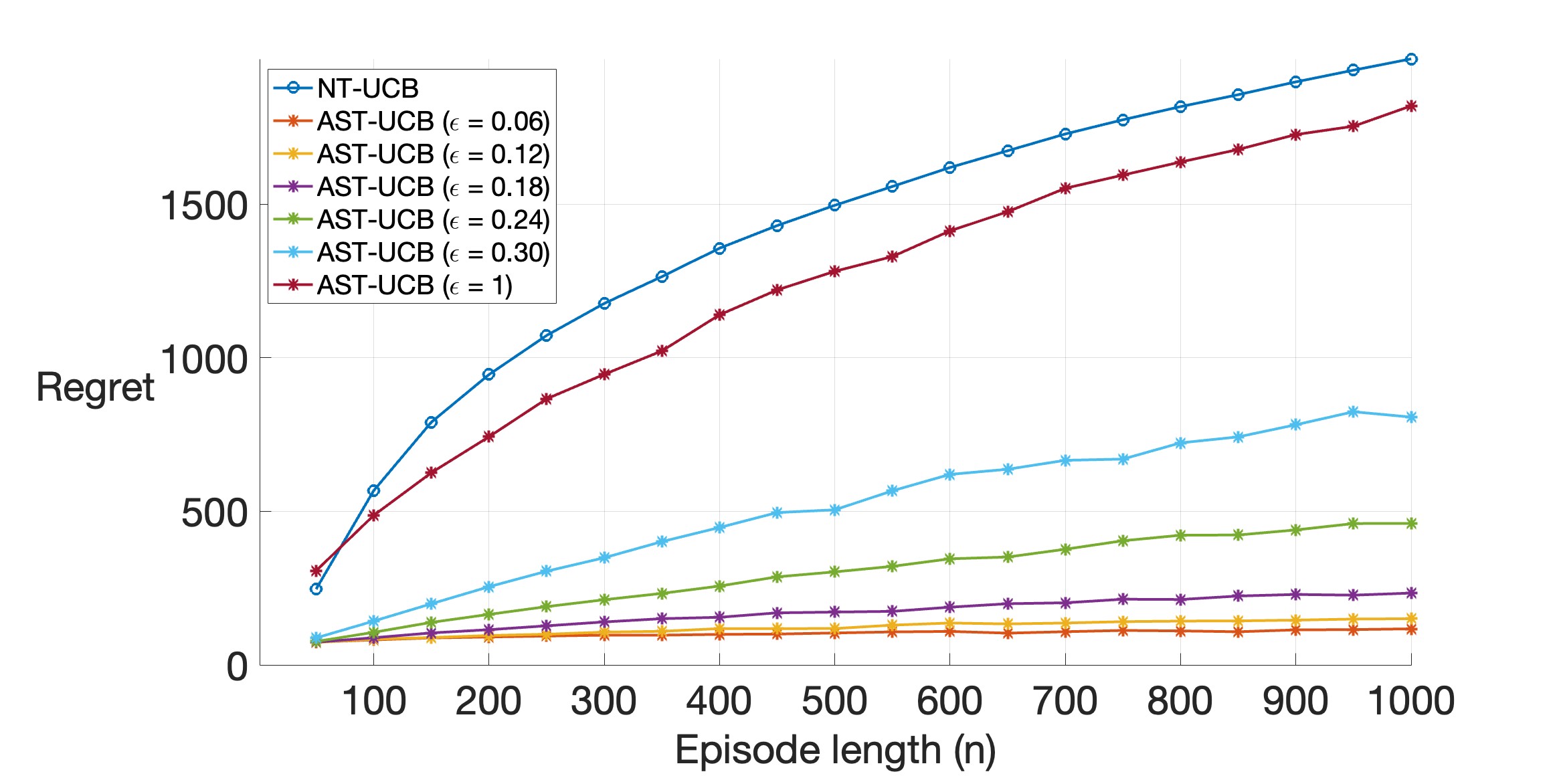}
    \caption{Regret as function of episode length $n$.}
    \label{fig:example_2a}
\end{subfigure}
\begin{subfigure}{\columnwidth}
    \centering
    \includegraphics[width=\columnwidth]{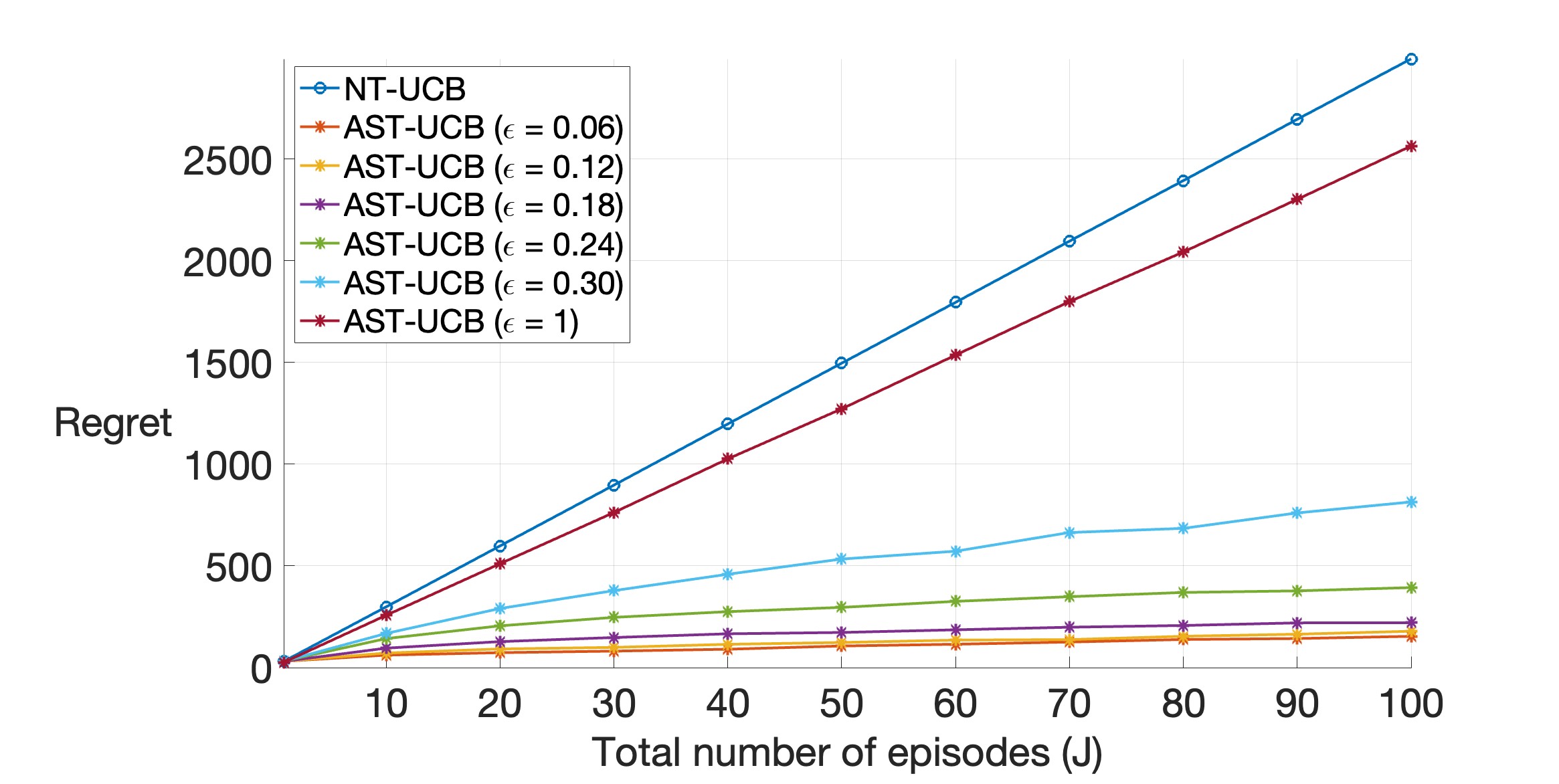}
    \caption{Regret as function of total number of episodes $J$.}
    \label{fig:example_2b}
\end{subfigure}
\caption{Empirical regret of NT-UCB and AST-UCB for different values of $\epsilon$ for Case II.}
\label{fig:example_2}
\end{figure}

In this section, we present the numerical results for AST-UCB algorithm. We consider $K = 4$ armed bandit problem. In numerical simulations, we need to select the mean reward ($\mu_k^j$) of each arm for each episode which should satisfy Assumption \ref{assump1}. Towards this end, we fix a seed interval of length $\epsilon$ for each arm. Then, at the beginning of each episode, we uniformly sample the value of $\mu_k^j$ from this seed interval. This ensures that Assumption \ref{assump1} is satisfied.
%In order to simulate the changing reward distribution across episodes, we sample the means of the arms uniformly from fixed intervals of width $\epsilon$ corresponding to each arm at the beginning of every episode. 
Once the mean reward value $\mu_k^j$ is obtained, we construct a uniform distribution with mean $\mu_k^j$ and width $d = 0.2$. In case the support of this uniform distribution lies outside the interval $[0,1]$, we reduce $d$ to avoid this. The reward samples are then generated from the uniform distribution. For each scenario, we compute the regret $R_J$ by taking an empirical average over $30$ independent realizations of that scenario.  

We simulate AST-UCB and NT-UCB for two cases (two sets of seed intervals). Note that the seed intervals for each arm are of length $\epsilon$. The mid-points of the seed intervals of the four arms for Case I and Case II are $(0.4,0.6,0.6,0.4)$ and $(0.35,0.7,0.3,0.4)$, respectively. 
%Notice that the mean rewards of different arms are closer in Case I as compared to Case II. 

In Figure \ref{fig:example_1a}, we observe that the regret of AST-UCB is considerably smaller as compared to NT-UCB. This is particularly true for smaller values of $\epsilon$. As $\epsilon$ increases\footnote{Since the reward support is $[0,1]$, values of $\epsilon>1$ are not valid in our setting.}, the regret of AST-UCB approaches to that of NT-UCB. This is in accordance with the fact that when $\epsilon$ is large, the confidence interval of the auxiliary estimate in \eqref{eq:lemma3} is large and transfer is not much beneficial. Further, we observe a logarithmic dependence of the regret on $n$, as quantified by the regret bounds in \eqref{eq:ucbRegret} and \eqref{eq:astUcb_regret}.

In Figure \ref{fig:example_1b}, we again observe that AST-UCB performs better than NT-UCB, particularly for small values of $\epsilon$. We also observe that the regret has a ``approximate" linear dependence on $J$. The plots in Figures \ref{fig:example_1} show that for any value of $\epsilon$ the difference between the regret of NT-UCB and AST-UCB increases with episode length ($n$) or total number of episodes ($J$). This is because a larger number of reward samples from previous episodes become available, thereby increasing the transfer. %increases then the reward samples available from the previous episodes for transfer increase, leading to improved regret.

Similar observations can be seen in Figures \ref{fig:example_2a} and \ref{fig:example_2b} for Case II. However, the improvement of AST-UCB over NT-UCB in terms of regret is more in Case II as compared to Case I. This happens because the seed intervals in Case II are farther apart, which helps in distinguishing the best arm more quickly using the samples of previous episodes.

\section{CONCLUSION}
We analyzed the transfer of reward samples in a sequential stochastic multi-armed bandit setting. We proposed a transfer algorithm based on UCB and showed that its regret is lower than UCB with no transfer. We provide regret analysis of our algorithm and validate our approach via numerical experiments. Future research directions include extending the work to the case when the parameter $\epsilon$ is unknown, and studying a similar transfer problem in the context of reinforcement learning.
 
% \addtolength{\textheight}{-12cm}   % This command serves to balance the column lengths
%                                   % on the last page of the document manually. It shortens
%                                   % the textheight of the last page by a suitable amount.
%                                   % This command does not take effect until the next page
%                                   % so it should come on the page before the last. Make
%                                   % sure that you do not shorten the textheight too much.

% %%%%%%%%%%%%%%%%%%%%%%%%%%%%%%%%%%%%%%%%%%%%%%%%%%%%%%%%%%%%%%%%%%%%%%%%%%%%%%%%

% %%%%%%%%%%%%%%%%%%%%%%%%%%%%%%%%%%%%%%%%%%%%%%%%%%%%%%%%%%%%%%%%%%%%%%%%%%%%%%%%

% %%%%%%%%%%%%%%%%%%%%%%%%%%%%%%%%%%%%%%%%%%%%%%%%%%%%%%%%%%%%%%%%%%%%%%%%%%%%%%%%
\section*{APPENDIX: Proof of Theorem \ref{theorem1}}
%Appendixes should appear before the acknowledgment.
To simplify the notation, we re-denote several variables as $\mu = \mu_k^j$, $\mu_* = \mu_*^j$, $\hat{\mu}_{1} = \hat{\mu}_{1k}^{j}(t-1)$, $\hat{\mu}_{1*} = \hat{\mu}_{1k_{*}^{j}}^{j}(t-1)$, $\hat{\mu}_{2} = \hat{\mu}_{2k}(t-1)$, $\hat{\mu}_{2*} = \hat{\mu}_{2k_{*}^{j}}(t-1)$, $t_j^n = (j-1)n+1$,
\begin{align*}
&p_{1} = \sqrt{\frac{\alpha\log{(t-t_j^n)}}{2N_k^{j}(t-1)}}, p_{1*} = \sqrt{\frac{\alpha\log{(t-t_j^n)}}{2N_{k_*^j}^{j}(t-1)}},\\
&p_{2} = \sqrt{\frac{\alpha\log{(t-t_j^n)}}{2S_k(t-1)}}+U_k^j(t)\epsilon,\\ 
&p_{2*} = \sqrt{\frac{\alpha\log{(t-t_j^n)}}{2S_{k_*^j}(t-1)}}+U_{k_*}^j(t)\epsilon,\\
&u_{1k}^{j} = \frac{2\alpha\log{(n)}}{(\Delta_k^{j})^2}, u_{2k}^{j} = \frac{2\alpha\log{(n)}}{(\Delta_k^{j}-2\epsilon)^2}.
\end{align*}
For arm $k$ to be pulled at time $t$ ($I_t = k$), at least one of the following five conditions should be true:
\begin{align}
\hat{\mu}_{1} - p_1 &> \mu, \label{eq:cond1} \\
\hat{\mu}_{1*} + p_{1*} &\leq \mu_{*},
\label{eq:cond2}\\
\hat{\mu}_{2*} + p_{2*} &\leq \mu_{*}, \label{eq:cond3}\\
\hat{\mu}_{2} - p_2 & > \mu, \label{eq:cond4}
\end{align}
\begin{align}
\sqrt{\frac{\alpha\log{n}}{2N_k^{j}(t-1)}} & > \frac{\Delta_k^j}{2}\hspace{0.4cm}\text{and}
\hspace{0.4cm}\sqrt{\frac{\alpha\log{(n)}}{2S_k(t-1)}}+\epsilon >  \frac{\Delta_k^j}{2}. \label{eq:cond5}
\end{align}

We show this by contradiction. Assume that none of the conditions in \eqref{eq:cond1}-\eqref{eq:cond4} is true and the first condition in \eqref{eq:cond5} is false. Then, using the fact that $p_1< \sqrt{\frac{\alpha\log{n}}{2N_k^{j}(t-1)}} $, we have
{\begin{align} 
\hat{\mu}_{1*} + p_{1*} &> \mu_{*} =\Delta_k^j + \mu
\geq2p_{1}+\mu \geq\hat{\mu}_{1} + p_{1}, \label{eq:falseCond1} \\
\hat{\mu}_{2*} + p_{2*}&> \mu_{*} 
=\Delta_k^j + \mu \geq2p_{1}+\mu
\geq\hat{\mu}_{1} + p_{1}. \label{eq:falseCond2}
\end{align}}
Conditions in \eqref{eq:falseCond1} and \eqref{eq:falseCond2} imply
{\begin{align} 
\label{eq:falseMainCond1}
\min\{\hat{\mu}_{1*} + p_{1*},\hat{\mu}_{2*} + p_{2*}\}
>\hat{\mu}_{1} + p_{1}.
\end{align}}
Similarly, when none of the conditions in \eqref{eq:cond1}-\eqref{eq:cond4} is true and the second condition in \eqref{eq:cond5} is false, we get
{\begin{align} 
\label{eq:falseMainCond2}
\min\{\hat{\mu}_{1*} + p_{1*},\hat{\mu}_{2*} + p_{2*}\}
>\hat{\mu}_{2} + p_{2}.
\end{align}}
Thus, at least one of the conditions in \eqref{eq:falseMainCond1} and \eqref{eq:falseMainCond2} is true, and this yields
{\begin{align*}
\min\{\hat{\mu}_{1*} + p_{1*},&\hat{\mu}_{2*} + p_{2*}\}
>\min\{\hat{\mu}_{1} + p_{1},\hat{\mu}_{2} + p_{2}\}.
\end{align*}}
The above condition implies that the AST-UCB algorithm will not pull arm $k$, and hence, we have a contradiction. The cumulative regret after $J$ episodes (each with length $n$) is given by
{\begin{align*}
R_J &= \sum\limits_{j=1}^{J}\sum\limits_{k=1}^K \Delta_{k}^{j}\mathbb{E}[N_k^j (jn)],\\
&\leq \sum\limits_{k=1}^{K}\Delta_k^{\max}\mathbb{E}[\tilde{S}_k (Jn)],
\end{align*}}
where $\tilde{S}_k (Jn)$ is the total number of sub-optimal pulls to arm $k$ over all episodes.
Next, we bound the regret by bounding the term $\mathbb{E}[\tilde{S}_k (Jn)]$. For an arbitrary sequence $I_t$, $t=1,2,\cdots,Jn$, we have 
{\begin{align}
\label{eq:S_k}
\tilde{S}_k(Jn)&=\sum\limits_{j=1}^{J}\sum\limits_{t=t_j^n}^{jn}\mathds{1}\{ I_t = k;k\neq k_{*}^j\},\nonumber\\
&=\sum\limits_{j=1}^{J}\bigg(\mathds{1}\{k\neq k_{*}^j\}+\sum\limits_{t=t_j^n+K}^{jn}\mathds{1}\{ I_t = k;k\neq k_{*}^j\}\bigg),\nonumber\\
&=\sum\limits_{j=1}^{J}\sum\limits_{t=t_j^n+K}^{jn}\mathds{1}\{ I_t = k;k\neq k_{*}^j;\eqref{eq:cond5}\hspace{0.1cm} \text{is}\hspace{0.1cm} \text{True}\}\nonumber\\
&\hspace{0.5cm}+\sum\limits_{j=1}^{J}\bigg(\mathds{1}\{k\neq k_{*}^j\}+\sum\limits_{t=t_j^n+K}^{jn}\mathds{1}\{ I_t = k;k\neq k_{*}^j;\nonumber\\
&\hspace{1.4cm}\eqref{eq:cond5}\hspace{0.1cm} \text{is}\hspace{0.1cm} \text{False}\}\bigg).
\end{align}}
\vspace{-0.5cm}
{\begin{align*}
\label{eq:S_k_first}
&\text{First term in \eqref{eq:S_k}} = \sum\limits_{j=1}^{J}\sum\limits_{t=t_j^n+K}^{jn} \mathds{1}\{ I_t = k, k \neq k_*^{j};\\
&\hspace{3cm}N_k^{j}(t-1)<u_{1k}^{j},S_k(t-1)<u_{2k}^{j}\},\\
&= \sum\limits_{j=1}^{J}\sum\limits_{t=t_j^n+K}^{jn} \min\bigg\{\mathds{1}\{ I_t = k, k \neq k_*^{t};N_k^{j}(t-1)<u_{1k}^{j}\},\\
&\hspace{1cm}\mathds{1}\{ I_t = k, k \neq k_*^{t};
S_k(t-1)<u_{2k}^{j}\}\bigg\},
\end{align*}}
{\begin{align}
&\leq \min\bigg\{\sum\limits_{j=1}^{J}\sum\limits_{t=t_j^n+K}^{jn} \mathds{1}\{ I_t = k, k \neq k_*^{t};N_k^{j}(t-1)<u_{1k}^{j}\},\nonumber\\
&\hspace{1cm}\sum\limits_{j=1}^{J}\sum\limits_{t=t_j^n+K}^{jn}\mathds{1}\{ I_t = k, k \neq k_*^{t};S_k(t-1)<u_{2k}^{j}\}\bigg\},\nonumber\\
&\leq\min\bigg\{\sum\limits_{\substack{j=1 \\
\Delta_{k}^{j}> 0}}^{J}\frac{2\alpha\log{(n)}}{(\Delta_k^{j})^2},\sum\limits_{t=1}^{Jn}\mathds{1}\bigg\{ I_t = k, k \neq k_*^{t};\nonumber\\
&\hspace{1cm}S_k(t-1)<\frac{2\alpha\log{(n)}}{(\Delta_k^{\min}-2\epsilon)^2}\bigg\}\bigg\},\nonumber\\
&\leq\min\bigg\{\sum\limits_{\substack{j=1 \\
\Delta_{k}^{j}> 0}}^{J}\frac{2\alpha\log{(n)}}{(\Delta_k^{j})^2},\frac{2\alpha\log{(n)}}{(\Delta_k^{\min}-2\epsilon)^2}\bigg\},
\end{align}}
\vspace{-0.5cm}
{\begin{align}
\label{eq:S_k_second}
&\text{Second term of \eqref{eq:S_k}}\leq \sum\limits_{j=1}^{J}\bigg(1+\sum\limits_{t=t_j^n+K}^{jn} \mathds{1}\{\eqref{eq:cond1}\hspace{0.1cm} \text{or}\hspace{0.1cm} \eqref{eq:cond2}\nonumber\\
&\hspace{4cm}\text{or}\hspace{0.1cm}\eqref{eq:cond3}\hspace{0.1cm}\text{or}\hspace{0.1cm} \eqref{eq:cond4}\hspace{0.1cm} \text{is}\hspace{0.1cm} \text{True}\}\bigg).
\end{align}}
Using \eqref{eq:S_k}, \eqref{eq:S_k_first}, \eqref{eq:S_k_second} and taking expectation, we get
{\begin{align}
\label{eq:S_k_f+S}
\mathbb{E}[S_k (n)]&\leq \min\left\{W_J,V_J\right\}+ \sum\limits_{j=1}^{J}\bigg(1+\sum\limits_{t=t_j^n+K}^{jn} \text{Pr}\{\eqref{eq:cond1}\hspace{0.1cm}\text{or}\nonumber\\
&\hspace{1cm} \eqref{eq:cond2}\hspace{0.1cm} \text{or}\hspace{0.1cm} \eqref{eq:cond3}\hspace{0.1cm} \text{or}\hspace{0.1cm} \eqref{eq:cond4}\hspace{0.1cm} \text{is}\hspace{0.1cm} \text{True}\}\bigg),\\
\nonumber
\text{where}\hspace{0.1cm}W_J &= \sum\limits_{\substack{j=1 \\ \Delta_{k}^{j}> 0}}^{J}\frac{2\alpha\log{(n)}}{(\Delta_k^{j})^2}\hspace{0.1cm}\text{and}\hspace{0.1cm}V_J = \frac{2\alpha\log{(n)}}{(\Delta_k^{\min}-2\epsilon)^2}.
\end{align}}

Next, we bound the probability of the event that at least one of \eqref{eq:cond1} or \eqref{eq:cond2} or \eqref{eq:cond3} or \eqref{eq:cond4} is true. We use the union bound, followed by the application of one-sided Hoeffding's inequality (steps are similar to the proof of Lemma \ref{lemma1} and \ref{lemma3}) to get,
{\begin{align}
\label{eq:probBound}
&\text{Pr}\{\eqref{eq:cond1}\hspace{0.1cm}\text{or}\hspace{0.1cm} \eqref{eq:cond2}\hspace{0.1cm} \text{or}
\hspace{0.1cm} \eqref{eq:cond3}\hspace{0.1cm} \text{or}\hspace{0.1cm}  \eqref{eq:cond4}\hspace{0.1cm} \text{is}\hspace{0.1cm} \text{True}\}\nonumber\\
&\leq \text{Pr}\{\eqref{eq:cond1}\hspace{0.1cm} \text{is}\hspace{0.1cm} \text{True}\}+\text{Pr}\{\eqref{eq:cond2}\hspace{0.1cm} \text{is}\hspace{0.1cm} \text{True}\}+\text{Pr}\{\eqref{eq:cond3}\hspace{0.1cm} \text{is}\hspace{0.1cm} \text{True}\}\nonumber\\
&\hspace{0.5cm}+\text{Pr}\{\eqref{eq:cond4}\hspace{0.1cm} \text{is}\hspace{0.1cm} \text{True}\}\},\nonumber\\
&\leq \frac{4}{(t-t_j^n)^{\alpha}}.
\end{align}}
Using \eqref{eq:S_k_f+S} and \eqref{eq:probBound}, we have
{\begin{align*}
&\mathbb{E}[\tilde{S}_k (n)]\leq \min\left\{W_J,V_J\right\}
+\sum\limits_{j=1}^{J}\bigg(1+\sum\limits_{t=t_j^n+K}^{jn} \frac{4}{(t-t_j^n)^{\alpha}}\bigg),\\
&\leq \min\left\{W_J,V_J\right\}
+\sum\limits_{j=1}^{J}\bigg(1+\int\limits_{s=(j-1)n+K}^{\infty} \frac{4}{(s-t_j^n)^{\alpha}}ds\bigg),\\
&= \min\left\{W_J,V_J\right\}+ \sum\limits_{j=1}^{J}\bigg(1+\frac{4(K-1)^{1-\alpha}}{\alpha-1}\bigg),
\end{align*}}
{\begin{align*}
&\leq \min\left\{W_J,V_J\right\}+ \sum\limits_{j=1}^{J}\bigg(1+\frac{4}{\alpha-1}\bigg),\\
&\leq \min\left\{W_J,V_J\right\}+ J\frac{\alpha+3}{\alpha-1}.
\end{align*}}
Hence the theorem follows.

\section*{ACKNOWLEDGMENT}

% The preferred spelling of the word ÒacknowledgmentÓ in America is without an ÒeÓ after the ÒgÓ. Avoid the stilted expression, ÒOne of us (R. B. G.) thanks . . .Ó  Instead, try ÒR. B. G. thanksÓ. Put sponsor acknowledgments in the unnumbered footnote on the first page.

%%%%%%%%%%%%%%%%%%%%%%%%%%%%%%%%%%%%%%%%%%%%%%%%%%%%%%%%%%%%%%%%%%%%%%%%%%%%%%%%

% References are important to the reader; therefore, each citation must be complete and correct. If at all possible, references should be commonly available publications.

\bibliographystyle{ieeetr}
\bibliography{reference}

\end{document}